\documentclass{article}[11pt,onecolumn,letter]
\usepackage{fixltx2e}

\usepackage{amsmath,amssymb,amsthm,amsfonts}
\usepackage{bm}
\usepackage{fullpage}
\usepackage{graphicx}
\usepackage{verbatim}
\usepackage{algorithm}
\usepackage[noend]{algorithmic}
\usepackage{subfigure}
\usepackage{url}

\usepackage{authblk,url}

\newcommand{\E}{\mathsf{E}}
\DeclareMathOperator*{\Prob}{\mathsf{P}}
\DeclareMathOperator*{\st}{^{\text{st}}}
\DeclareMathOperator*{\nd}{^{\text{nd}}}
\DeclareMathOperator*{\nth}{^{\text{th}}}

\newcommand{\id}{\mathrm{id}}

\newcommand{\real}{\mathbb{R}}

\newcommand{\hilb}{\mathcal{H}}
\newcommand{\X}{\mathcal{X}}
\newcommand{\Xp}{\mathcal{X'}}
\newcommand{\C}{\mathcal{C}}
\newcommand{\R}{\mathcal{R}}

\newcommand{\q}{\mathbf{q}}
\newcommand{\qp}{\mathbf{q'}}
\newcommand{\x}{\mathbf{x}}
\newcommand{\xp}{\mathbf{x'}}
\newcommand{\mmk}{k_{\mathrm{m}}}
\newcommand{\ppk}{k_{\mathrm{p}}}

\newtheorem{theorem}{Theorem}

\newtheorem{lemma}{Lemma}
\newtheorem{proposition}{Proposition}

\begin{document}

\title{Generative and Latent Mean Map Kernels}
\author{Nishant A. Mehta$^*$}
\author{Alexander G. Gray}
\affil{
  School of Computer Science \\
  College of Computing \\
  Georgia Institute of Technology \\
  Atlanta, Georgia
}
\date{\today}
\maketitle
\thispagestyle{empty}
\let\oldthefootnote\thefootnote
\renewcommand{\thefootnote}{\fnsymbol{footnote}}
\footnotetext[1]{To whom correspondence should be addressed. Email: \protect\url{niche@cc.gatech.edu}}
\let\thefootnote\oldthefootnote

\maketitle

\maketitle 
\begin{abstract}
We introduce two kernels that extend the mean map, which embeds probability measures in Hilbert spaces.
The generative mean map kernel (GMMK) is a smooth similarity measure between probabilistic models. 
The latent mean map kernel (LMMK) generalizes the non-iid formulation of Hilbert space embeddings of empirical distributions in order to incorporate latent variable models.
When comparing certain classes of distributions, the GMMK exhibits beneficial regularization and generalization properties not shown for previous generative kernels.
We present experiments comparing support vector machine performance using the GMMK and LMMK between hidden Markov models to the performance of other methods on discrete and continuous observation sequence data. The results suggest that, in many cases, the GMMK has generalization error competitive with or better than other methods.
\end{abstract}

\section{Introduction} \label{sec:intro}
Generative kernels offer an elegant way to apply kernel methods for classification, clustering, and manifold learning to distributions, and they are particularly useful for applying kernel methods to non-iid data. By using kernels that incorporate statistical dependence information, we can leverage a rich set of existing methods such as kernel support vector machines (SVMs) and kernel principal components analysis (kPCA) \cite{scholkopf1997kpc} to learn from this data. As an example, we consider sequence classification as a particular instance of learning from non-iid data. In sequence classification, the goal is to label a sequence $(X_1, X_2,  \ldots, X_T)$ with one label $Y$, where the example sequences can be of varying lengths and the ordering of the observations within a sequence informs non-trivial dependencies.

SVMs using nonlinear kernels such as the polynomial and Gaussian kernels have performed strongly for a variety of non-sequential data classification tasks \cite{scholkopf2002lk}. More recent work has applied kernels to measure similarity between sequences via similarity between generative models trained on those sequences \cite{jebara2004ppk} or by making use of
metrics on statistical manifolds \cite{jaakkola1999egm,lafferty2005dks}. Guilbart \cite{guilbart1979pse} and later Suquet \cite{suquet1995dem} (also see \cite{suquet2008rkh} for a related English-language paper) previously extended reproducing kernels to kernels on bounded signed measures. One particular kernel of Hein and Bousquet \cite{hein2005hmp}, Structural Kernel I, generalizes the kernel of Suquet. Unfortunately, efficient computation of this kernel family for many interesting distributions is highly non-trivial.

We provide a new derivation for the natural special case of Structural Kernel I described by Guilbart and Suquet, provide concrete examples for how it can be computed for several distributions relevant to the machine learning community, and provide for the first time generalization error guarantees when learning with this kernel.
We also introduce a second kernel that naturally extends the idea of Hilbert space embeddings of empirical clique distributions (the empirical mean map) to graphical models with latent variables.

For lack of a generally accepted name for this kernel, we refer to it in this work as the generative mean map kernel (GMMK). The GMMK measures similarity between observations by providing a nonlinear similarity between generative models estimated from each observation. We show in Section \ref{sec:theory} that the GMMK has unique learning theoretic advantages.
We also introduce the latent MMK (LMMK), which measures similarity between pairs of structured data observations with respect to a single (global) generative model $\theta$. This is accomplished by measuring the similarity of sufficient empirical and posterior distributions from two structured data observations.

We begin by reviewing Hilbert space embeddings of distributions via the mean map and then introduce the generative mean map kernel.
In Section \ref{sec:generative-MMK} we show how to compute the GMMK for several widely-used distributions,
and in Section \ref{sec:latent-MMK} we discuss the LMMK, an extension of the non-iid empirical MMK for latent variable models.
We form connections between these two kernels and other kernels in Section \ref{sec:relation}.
We then analyze some theoretical properties of the GMMK before 
concluding with promising results on
sequence classification and learning a species manifold from biodiversity data.

\section{The Mean Map} \label{sec:the-mean-map}
Although the concept dates back quite a bit \cite{guilbart1978eps,guilbart1979pse,suquet1995dem}, the phrase \textit{mean map} was coined by Smola et al. \cite{smola2007hse}  as a Hilbert space embedding of empirical distributions. For $\X$ a domain of observations with probability measure $P_x$ and $X = \{x_1, \ldots, x_m\} \subset \X$ a set of $m$ samples drawn iid, consider the reproducing kernel Hilbert space (RKHS) $\hilb$ with feature map $\phi: \X \rightarrow \hilb$, where $\phi(x) = k(x, \cdot)$ and from the reproducing property we have $\langle f, \phi(x) \rangle = f(x)$ and $\langle \phi(x), \phi(y) \rangle = k(x, y)$.
The mean map $\mu$ of the true ($P_x$) and empirical ($X$) distributions respectively are defined as
\begin{align}
\mu[P_x] :&= \E_x[k(x, \cdot)] \\
\mu[X] :&= \frac{1}{m} \sum_{i=1}^m k(x_i, \cdot).  \label{eqn:mean-map}
\end{align}
The operator $\mu$ maps distributions to elements of the RKHS and is injective for RKHSs induced by universal kernels \cite{steinwart2002ikc}, such as
the Gaussian radial basis function (RBF) kernel $k(x,y) = \exp(-\lambda \| x - y \|^2)$ and the Laplace kernel $k(x,y) = \exp(-\beta \sum_{i=1}^n |x_i - y_i|)$ for $\lambda,\beta \in \real^+$.

Some previous works \cite{gretton2005msd,zhang2009kmi} exploited the linear convergence of $\mu[X]$ to $\mu[P_x]$ in order to compute kernels between two sets of samples; the methods introduced here are departures from these applications of the mean map.
After briefly describing the empirical mean map kernel (EMMK) we connect it to the GMMK.

\subsection{Empirical mean map kernel}
The empirical mean map is an injective mapping (for universal base kernels) of empirical probability distributions into an RKHS. The method has been applied to iid observations and also non-iid observations \cite{zhang2009kmi} by fixing a dependency model for the observations and considering the empirical distributions induced by the maximal cliques of this model. The empirical mean map of non-iid data can be decomposed into the sum of the empirical mean maps for each maximal clique's distribution.

Following the notation from \cite{zhang2009kmi}, for a graphical model $G$ with variable set $Z$ and maximal cliques set $\C$, let $v_c$ be universal kernels on the variable subset of $Z$ induced by clique $c \in \C$.
Then
\begin{align} 
v(z,z') = \sum_{c \in \C} v_c(z_c, z'_c)
\end{align}
embeds all probability distributions with the specified conditional independence relations using an exponential family model with kernel $v$ \cite{altun2004efc}.

A limitation of the empirical mean map is that it operates only on observed data and hence cannot handle latent variable dependency models. Critically, using dependency models over only observable variables can be quite restrictive, as highlighted by the strong performance of latent variable models such as hidden Markov models (HMMs) for many learning problems. Also, note that although Zhang et al. \cite{zhang2009kmi} used the empirical mean map with latent variable models, the latent variables were used to optimize a kernel-based objective rather than to model the dependency explicitly.
Also, while the EMMK previously has been used to measure and optimize dependence \cite{gretton2005msd,zhang2009kmi,gretton2008kmt}, it has not yet been used for classification.

\section{The Generative Mean Map Kernel} \label{sec:generative-MMK}

\subsection{The generative mean map}
Suppose we are given two objects $x$ and $y$. These objects could be documents, sequences, images, or 
points in $\real^n$. For all but the last case,
special effort is necessary to form a suitable kernel between the objects that captures their underlying similarity well.
We can modify the mean map in \eqref{eqn:mean-map} such that the expectation is evaluated for $x \sim \hat{P_x}$ rather than $x \sim P^n$ (for $P^n$ the empirical distribution induced by the sample $X$), where $\hat{P_x}$ is an estimated probabilistic model of $x$. This modification induces the generative mean map
\begin{align}
\mu[\hat{P_x}] = \E_{x \sim \hat{P_x}}[\phi(x)].
\end{align}
For $\hat{P}_x$ learned from $x$ and $\hat{P}_y$ learned from $y$, let the \textit{generative mean map kernel} be
\begin{align}
  \langle \mu[\hat{P}_x], \mu[\hat{P}_y] \rangle 
  &= \E_{x \sim \hat{P}_x, y \sim \hat{P}_y}[k(x, y)] \\
  &= \int \int \hat{P}_x(x) \hat{P}_y(y) k(x, y) \, dx \, dy. \label{eqn:gmmk}
\end{align}
Though it has already been shown in various other works that this kernel is pd \cite{guilbart1978eps,guilbart1979pse,suquet1995dem,hein2005hmp}, for completeness we provide a short proof.
\begin{proposition}
The generative mean map kernel with positive definite (pd) kernel $k$ (i.e. $k \succeq 0$) also is pd.
\end{proposition}
\begin{proof} 
For $k$ pd, there exists a feature map $\phi: \X \rightarrow \hilb$.
The GMMK is an inner product between mean elements in $\hilb$, for mean elements $\mu[p] = \int p(x) \phi(x) dx$.
By identification of the inner product between mean elements, given in \eqref{eqn:gmmk}, the GMMK is pd.
\end{proof}
Note that the EMMK is a kernel between \textit{sets of points} (which induce empirical distributions) whereas the GMMK is a kernel between \textit{functions}. We explore this kernel for examples of generative models of increasing structure. 

\subsection{Examples}

\subsubsection{Discrete distribution}
For $p$ and $p'$ discrete distributions with mean parameters $\alpha = (\alpha_1, \ldots, \alpha_k)$ and $\alpha' = (\alpha'_1, \ldots, \alpha'_k)$ respectively, the GMMK is
\begin{equation}
  \mmk(p, p') = \sum_{i = 1}^k \sum_{j=1}^k \alpha_i \alpha'_j \exp(-\lambda (1 - \delta_{i j})),
\end{equation}
where $\delta_{i,j} = 1$ if $i = j$ and 0 otherwise.

It would be of considerable benefit to compute this kernel for multinomial distributions; it currently is open whether this case admits a closed form expression.

\subsubsection{Gaussian distribution}
\begin{lemma}
  Let $p$ and $p'$ be multivariate Gaussian probability measures
$\mathcal{N}(\mu, \Sigma)$ and $\mathcal{N}(\mu', \Sigma')$ respectively.
  For the Gaussian RBF kernel $k(x, x') = \exp(-\frac{1}{2}\lambda \|x - x'\|^2)$, the GMMK for $p$ and $p'$ is
  \begin{align*}
    \mmk(p,p') = 
    \int \int
    \frac{\exp(-\frac{1}{2}(x - \mu)^T \Sigma^{-1} (x - \mu))}{(2 \pi)^{d/2} |\Sigma|^{1/2}} \,\,
    \frac{\exp(-\frac{1}{2}(x' - \mu')^T \Sigma'^{-1} (x' - \mu'))}{(2 \pi)^{d/2} |\Sigma'|^{1/2}} \,\,
    k(x, x') \, dx \, dx'
  \end{align*}

 \begin{flalign}
  \label{eqn:general-Gaussian-GMMK}
  \text{which can be computed in closed form as } \qquad
  \frac{\exp(-\frac{1}{2} (\beta^T \alpha^{-1} \beta - \delta))}{|I + \lambda (\Sigma + \Sigma')|^{1/2}}, && 
  \end{flalign}
\begin{flalign*}
\text{where\footnotemark } \qquad \alpha &= \Sigma^{-1} (\Sigma^{-1} + \lambda I)^{-1} \Sigma^{-1} + \Sigma'^{-1} + \Sigma^{-1}, &&\\
\beta &= \lambda \Sigma^{-1} (\Sigma^{-1} + \lambda I)^{-1} \mu + \Sigma'^{-1} \mu', &&\\
\delta &= -\lambda^2 \mu^T (\Sigma^{-1} + \lambda I)^{-1} \mu + \mu'^T \Sigma'^{-1} \mu' + \lambda \mu^T \mu. &&
\end{flalign*}
\footnotetext{We have confirmed that the exponential term is symmetric in $p$ and $p'$. See the simpler isotropic case in \eqref{eqn:isotropic-Gaussian-GMMK}.}

\end{lemma}
The proof follows from a multitude of linear algebra identities.

\subsubsection{Markov models and the connection to EMMK}
Suppose that each observation is a sequence of elements of $\real$:
\[
x_i = (x_i^{(1)}, \ldots, x_i^{(T_i)}) \qquad \text{for all } i \in [n].
\]
For both the generative mean map and the non-iid extension of the empirical mean map, let us assume the first-order Markov dependency model
\begin{align}
P \left( x_i^{(1)}, \ldots, x_i^{(T_i)} \right) = \prod_{t=1}^{T_i} P \left( x_i^{(t)} | x_i^{(t-1)} \right).
\end{align}

For the empirical mean map, we make no assumption on the form of $P^{(i)}_{t | t-1} := P \left( x_i^{(t)} | x_i^{(t-1)} \right)$, whereas for the generative mean map, we explicitly estimate $P^{(i)}_{t | t-1}$ as $\hat{P}^{(i)}_{t | t - 1}$.

The empirical mean map is
\begin{align}
\frac{1}{T_i-1} \sum_{t=1}^{T_i - 1} \phi \left( \left( x_i^{(t)}, x_i^{(t+1)} \right) \right),
\end{align}
whereas the generative mean map is
\begin{align}
\E_{(x_1, \ldots, x_T) \sim \hat{P}^{(i)}}[\phi((x_1, \ldots, x_T))] = \int \prod_{t=1}^T \hat{P}^{(i)}_{t | t - 1} \phi((x_1, \ldots, x_T)) \, dx_1 \ldots \, dx_T
\end{align}
for $T$ a free parameter. The kernel for each map is simply the inner product in that map's feature space. We now show by example that the GMMK can be computed efficiently for various graphical models.

\subsubsection{Hidden Markov models}
For a hidden Markov model (HMM) with probability measure $p$, let $\q = (q_0, \ldots, q_T)$ be the latent random variables and $\x = (x_0, \ldots, x_T)$ be the observable random variables. We similarly define $p'$, $\qp$, and $\xp$ for a second HMM.

Suppose that we have learned HMMs with probability measures $p$ and $p'$ and wish to compute the GMMK $\mmk(p,p')$ for the observable variables of length $T$ sequences drawn from these HMMs. The parameter $T$ serves as a witness length which allows control over the length of the sequences to be embedded in the RKHS (a larger $T$ translates to less weight on the models' initial conditions). 
The next result establishes the complexity of an efficient algorithm to compute the kernel.
\begin{lemma}
  The generative mean map kernel between an HMM of $n$ states with probability measure $p$ and an HMM of $n'$ states with probability measure $p'$ can be computed in time:
  \begin{enumerate}
  \item
   $O(n^3 T + n^2 k^2)$
    for a discrete observation HMM on $k$ symbols with $n' = n$.
  \item
   $O(n^3 T + n^2 m^2 \rho_d)$
    for a continuous observation HMM with mixture of Gaussians state distributions, $n'=n$, and $m'=m$, for $d$ the observation dimensionality, $m$ and $m'$ the number of Gaussians in the mixtures, and $\rho_d$ the cost of inverting a covariance matrix ($\rho_d=d$ for diagonal case).
 \end{enumerate}
\end{lemma}
\begin{proof}
  The quantity of interest is
  \begin{align}
    \mmk(p, p') = \sum_{\x, \xp, \q, \qp} p(\q, \x) p'(\qp, \xp) k(\x,\xp).
  \end{align}
  We treat the discrete and continuous observation cases simultaneously in the following way. For the discrete case, we use the 1-of-$k$ encoding wherein, if the $t\nth$ observation takes on value $i$ out of $k$ possible values, then $x_t \in \{0, 1\}^k$ and $[x_t]_j = \delta_{i j}$, for all $j \in [k]$. In the derivation for the discrete case below\footnote{For continuous observations, sums become integrals.}, we use the Gaussian RBF base kernel's isotropicity such that the kernel factorizes as $k(\mathbf{x}, \mathbf{x'}) = \prod_{t = 1}^T k(x_t, x'_t)$. Also, the linear chain structure of the HMM graphical model can be used to factorize the expectation
  \begin{align}
    \mmk(p,p')
    = {} & \sum_{q_T, x_T} p(x_T \mid q_T) \sum_{q'_T, x'_T} p'(x'_T \mid q'_T) k(x_T, x'_T)
    \nonumber \\
    & \prod_{t=0}^{T-1} \sum_{q_t, q'_t} p(q_{t+1} \mid q_t) p'(q'_{t+1} \mid q'_t)
    \sum_{x_t, x'_t} p(x_t \mid q_t) p'(x'_t \mid q'_t) k(x_t,x'_t)
    p(q_0) p'(q'_0)
    \nonumber \\
    = {} & \sum_{q_T, q'_T} \psi(q_T, q'_T)
    \prod_{t=0}^{T-1} \sum_{q_t} p(q_{t+1} \mid q_t)
    \sum_{q'_t} p'(q'_{t+1} \mid q'_t) \psi(q_t, q'_t)
    p(q_0) p'(q'_0).
    \label{eqn:last-line}
  \end{align}
  Now, 
  $\psi(q_t, q'_t) = \sum_{x_t, x'_t} p(x_t \mid q_t) p'(x'_t \mid q'_t) k(x_t, x'_t)$
  is itself a GMMK on the state distributions for $q_t$ and $q'_t$. These kernels need be computed only once for each pair of states $(q_t, q'_t)$, yielding cost
  $n^2$
  times the complexity to evaluate this kernel once, which is $O(k^2)$ for the discrete distribution and
  $O(m^2 \rho_d)$
  for a mixture of Gaussians distribution.
  From the factorized structure of the rest of the computation in \eqref{eqn:last-line} we see that it is
  $O(T n^3)$ (see the algorithm in Figure \ref{fig:mmk-hmm}),
  as all $O(T)$ latent variable marginalizations are done over functions of at most 3 variables. 
\end{proof}
\begin{figure}[ht]
  \fbox{
    \begin{minipage}{0.94\linewidth}
      \begin{algorithmic}
        \FOR{$i = 1$ to $n$}
        \FOR{$j = 1$ to $n'$}
        \STATE $\bm{\psi}_{j i} = \mmk(p(x | q=i), p(x' | q'=j))$
        \ENDFOR
        \ENDFOR
        \STATE $\bm{\phi} = \bm{\pi} \bm{\pi}'^T$
        \STATE $\bm{\phi} = \bm{\phi} \bullet \bm{\psi}$
        \FOR{$t=1$ to $T$}
        \STATE $\bm{\phi} = (\mathbf{A}'^T \bm{\phi} \mathbf{A}) \bullet \bm{\psi}$
        \ENDFOR
        \RETURN $\sum_{i = 1}^n \sum_{j = 1}^{n'} \bm{\phi}_{j i}$
      \end{algorithmic}
    \end{minipage}
  }
  \caption{\label{fig:mmk-hmm}
    \small We show how to compute the GMMK for HMMs. $\bullet$ is the Hadamard product, $[\mathbf{A}]_{i j} = P(q_{t+1} = j \mid q_t = i)$, and $[\bm{\pi}]_i = P(q_0 = i)$.
  }
\end{figure}

Generally, sequences for which one would like a similarity measure are of different lengths, precluding computation of a kernel without resorting to truncation or other compromises. Even for sequences of the same length, \textit{a priori} there is no reason why the sample indices of the sequences should be considered aligned; application of standard kernels invariably relies upon distance computations made between mismatched random variables.
The GMMK addresses this issue by first learning a generative model for each sequence and then performing kernel computations on the expected sequences that result from each generative model. While sequences drawn from similar distributions can appear to be very different due to the stochastic nature of their generation, by using a measure between the distributions of sequences themselves, we bypass this problem and achieve a more robust similarity measure.

\subsubsection{Linear dynamic systems}
The GMMK also can be computed for linear dynamic systems of the form
\begin{align}
q_{t+1} &= A q_t + w_t & w_t &\sim \mathcal{N}(\mathbf{0}, \mathbf{I})\footnotemark \\
x_t      &= C q_t + v_t  & v_t &\sim \mathcal{N}(\mathbf{0}, R),
\end{align}
\footnotetext{Note that we can use identity covariance ($\mathbf{I}$) without loss of generality (as per footnote 4 of \cite{roweis1999url}).}
where $A: \real^k \rightarrow \real^k$ and $C: \real^k \rightarrow \real^n$ are linear operators and $R $ is a covariance matrix.

The computation follows almost directly from the formulation of Jebara and Kondor \cite{jebara2004ppk} for the probability product kernel (described in Section \ref{sec:relation}). Briefly, for $p$ and $p'$ being probability measures over linear dynamic systems, $\mmk(p,p')$ can be shown to be the GMMK between two Gaussians. In particular, the two Gaussians are
$\mathcal{N}(\mu_x, \Sigma_{xx}) \text{ and } \mathcal{N}(\mu_{x'}, \Sigma_{x'x'})$,
where $\mu_x = (\mu_{x_0}, \ldots \mu_{x_T})$, $\Sigma_{xx}$ is a block diagonal matrix with blocks $\Sigma_{x_t,x_t}$, and we have the following recursive updates:
\begin{align*}
\mu_{q_0} := \mu && \mu_{q_t} &:= A \mu_{q_t}   &&& \Sigma_{q_{t+1},q_{t+1}} &:= A \Sigma_{q_t,q_t} A^T + I \\
                           && \mu_{x_t} &:= C \mu_{q_t}   &&& \Sigma_{x_t,x_t} &:= C \Sigma_{q_t,q_t} C^T + R.
\end{align*}

\subsubsection{Kernel density estimators}
The GMMK can be used as a kernel on nonparametric density estimators. The idea of kernels between density models of sets has been explored previously by Jebara and Kondor \cite{kondor2003kbs} with the Bhattacharyya kernel . Whereas they implicitly map the data to an RKHS using the Gaussian kernel and then learn single Gaussian models in the feature space, here we use kernel density estimators (KDEs) in the original space. An advantage of using kernel density estimation is that it is known to be consistent \cite{tsybakov2008ine}.

Let $\hat{f_z}(z)$ be a KDE
\begin{align}
\hat{f_z}(x) = \frac{1}{m_z} \sum_{i=1}^{m_z} k_{h_x}(z_i, x),
\end{align}
where $h_x$ is the bandwidth.
The GMMK between two Gaussian RBF KDEs $\hat{f_x}$ on observations $X = (x_1, \ldots, x_{m_x})$ and $\hat{f_y}$ on observations $Y = (y_1, \ldots, y_{m_y})$ is then
\begin{align}
\langle \mu[\hat{f}], \mu[\hat{f}'] \rangle
= \E_{ \substack{x \sim \hat{f} \\ x' \sim \hat{f}'} }[k(x,x')]
= \frac{1}{m_x} \frac{1}{m_y} \sum_{i=1}^{m_x} \sum_{j=1}^{m_y} \int k_{h_x}(x_i, x)
\int k_{h_y}(y_j, x') k(x,x') \, dx \, dx'. \label{eqn:KDE-kernel-integrals} 
\end{align}

For $k$ the Gaussian RBF kernel, this expression only requires evaluations of the GMMK on isotropic Gaussians. The form for general Gaussians is in \eqref{eqn:general-Gaussian-GMMK}. For two $N$-dimensional isotropic Gaussians  $\mathcal{N}(\mu, h I)$ and $\mathcal{N}(\mu', h' I)$, the GMMK admits the more pleasant form
\begin{align}
\frac{1}{(1 + \lambda (h + h'))^{N/2}} \exp \left( -\frac{1}{2} \frac{\lambda \| \mu - \mu' \|^2}{1 + \lambda (h + h')} \right)
= \frac{1}{h_0^{N/2}} \exp \left(-\frac{1}{2} \frac{\lambda \| \mu - \mu' \|^2}{h_0} \right) \label{eqn:isotropic-Gaussian-GMMK},
\end{align}
where $h_0 := 1 + \lambda (h + h')$.
Substituting \eqref{eqn:isotropic-Gaussian-GMMK} for the integrals in \eqref{eqn:KDE-kernel-integrals} yields KDE GMMK closed form
\begin{equation}
\frac{1}{m_x m_y h_0^{N/2}}
\sum_{i=1}^{m_x} \sum_{j=1}^{m_y}
 \exp \left(- \frac{1}{2} \frac{\lambda \|x _i- y_i\|^2}{h_0} \right).
\end{equation}

\section{The Latent Mean Map Kernel} \label{sec:latent-MMK}
\subsection{Empirical mean map limitations}
By relaxing the dependency models of the empirical mean map kernel to include latent variables, we generalize the kernel to include richer dependency models such as dynamic Bayesian networks
and hidden Markov random fields. As with the non-iid empirical mean map, we need only apply the mean map to the distribution of each maximal clique. The maximal cliques now fall into two sets: fully observable cliques and cliques containing at least one latent random variable. The distributions for the former cliques can be computed empirically similar to \cite{zhang2009kmi}; however, applying the empirical mean map to the distributions of the latter cliques is impossible due to the latent variables.

\subsection{The Latent mean map}
The latent mean map augments the empirical mean map by using the posterior distribution of the latent variables (with respect to a model specified by $\theta$), conditional on the observed variables. For conciseness, all expectations in this section implicitly are made with respect to a single model $\theta$; this $\theta$ should be estimated from the examples, or a subset of the examples, that are being embedded into a Hilbert space.

Let $(u,v)$ be the concatenation of the components of vectors $u$ and $v$ to form a higher dimensional vector. For observed variables $X$, latent variables $Y$, and clique-restricted subsets $X_c$ and $Y_c$ , the latent mean map of $(X_c,Y_c)$ is
\begin{align}
\mu_c[(X_c, Y_c)]
= \E_{(X_c,Y_c) \mid x_{1:m}} \left[ \phi_c((X_c, Y_c)) \right]
= \frac{1}{m_c} \sum_{i=1}^{m_c} \E_{Y_c^{(i)} \mid x_{1:m}} \left[ \phi_c \left( \bigl( x_c^{(i)}, Y_c^{(i)} \bigr) \right) \right] \label{eqn:latent-mean-map} \\
\end{align}
for
$Y_c^{(i)} \sim P(Y_c^{(i)} \mid x_{1:m})$, the posterior distribution of the random variable $Y_c^{(i)}$ conditioned on \textit{all} observations $x_{1:m}$.
This expression captures our best estimate of the clique distribution.

From \eqref{eqn:latent-mean-map}, the latent mean map kernel
\begin{align}
\sum_{c \in \C} \langle \mu_c[(X_c, Y_c)], \mu_c[(X'_c, Y'_c)] \rangle
\end{align}
expands to
\begin{align}
\sum_{c \in \C}
\frac{1}{m_c m'_c}
\sum_{i=1}^{m_c} \sum_{j=1}^{m'_c}
\E_{\substack{Y_c^{(i)} \mid x_{1:m} \\ Y_c'^{(j)} \mid x'_{1:m'}}}
\left[
v_c \left( \bigl( x_c^{(i)}, Y_c^{(i)} \bigr), \bigl( x_c'^{(j)}, Y_c'^{(j)} \bigr) \right)
\right].
\end{align}

Our end goal is to compute the kernel on many object pairs for SVM classification or kPCA, but even moderate $m_c$ and $m'_c$ render the above expectation intractable. An often exploited trick of kernel methods is the ability to compute inner products in a potentially infinite dimensional space without the need for explicit representations in that space. Here, however, an approximate explicit representation
yields computational tractability by allowing us to work with the efficient form in \eqref{eqn:latent-mean-map}.
For example, the Gaussian RBF kernel on univariate continuous data admits a truncated Taylor expansion of the exponential \cite[Theorem 4.6]{steinwart2006edr}, empirically yielding low error for low order truncations \cite{xu2006ecr}. For multivariate data, two recent explicit representations approximate the RKHS using random features, with error decreasing exponentially in the number of features chosen \cite{rahimi2007rfl}.

It may be useful to use nonlinear representations even in the space of distributional Hilbert space embeddings.
Given a
latent mean map kernel matrix $K$, this can be accomplished by using the alternate kernel matrix $\tilde{K}$ such that
\begin{align}
\tilde{K}_{i j} = \exp(-\nu (K_{i i} - 2 K_{i j} + K_{j j})),
\end{align}
for some parameter $\nu$.
In our LMMK experiments, we push $\lambda$ toward infinity and consider different values of $\nu$ rather than $\lambda$.

\subsection{Latent mean map of HMMs}
Learning using the latent MMK requires a dependency model to induce latent variables and a set of conditional distributions sufficient for the model. This model identifies a set of maximal cliques and
allows us to compute the posteriors.
Suppose we have an HMM $\theta$
as described earlier.
Assuming stationarity, the model's maximal cliques $(x_t, q_t)$ and $(q_t, q_{t+1})$ yield $T$ instances of the former and $T-1$ instances of the latter clique:
\begin{align*}
\mu_{xq}[(x_t, q_t)]
&= \frac{1}{T} \sum_{i=1}^{T}
\E_{Q_t \mid x}
\left[
\phi_{xq}((x_t, Q_t))
\right]
= \frac{1}{T} \sum_{t=1}^{T}
\sum_{i=1}^N \Prob(Q_t = i \mid x) \phi_{xq}((x_t, Q_t))
\\
\mu_{qq}[(q_t, q_{t+1})]
&= \frac{1}{T-1} \sum_{i=1}^{T-1}
\E_{Q_t, Q_{t+1} \mid x, \theta}
\left[
\phi_{qq}((Q_t, Q_{t+1}))
\right]
\\
&=\frac{1}{T-1} \sum_{t=1}^{T-1}
\sum_{i=1}^N \sum_{j=1}^N
\Prob(Q_t = i, Q_{t+1} = j \mid x, \theta)
\phi_{qq}((Q_t, Q_{t+1})).
\end{align*}
The forward-backward algorithm can compute the conditional probabilities \cite{rabiner1989thm}. We adopt Rabiner's notation \cite{rabiner1989thm} for the conditional probabilities so that we have
\begin{align}
\gamma_t(i) = \Prob(Q_t = i \mid x, \theta),
\end{align}
\begin{align}
\xi_t(i,j) = \Prob(Q_t = i, Q_{t+1} = j \mid x, \theta).
\end{align}
We use the joint kernel on cliques
$v_c((x_c, y_c),(x_c', y_c')) = k_c(x_c, x_c') l_c(y_c, y_c')$
such that
\begin{align}
v_{xq}((x_t, q_t), (x'_t, q'_t)) = k_x(x_t, x'_t) l_q(q_t, q'_t),
\end{align}
and likewise
\begin{align}
v_{qq}((q_t, q_{t+1}), (q'_t, q'_{t+1})) = l_q(q_t, q'_t) l_q(q_{t+1},q'_{t+1}).
\end{align}
As before, we use the 1-of-$k$ encoding to treat the discrete and continuous cases identically. The kernel $k$ will be the Gaussian RBF kernel.
For an $N$-state latent space, $K$ possible symbols, and $\gamma_{(c)}(j) := \sum_{t:x_t = c}^T \gamma_t(j)$, the kernel $v_{xq}$ is
\begin{align}
\frac{1}{T T'}
\sum_{\substack{a \in [K] \\ i \in [N]}}
\gamma_{(a)}(i)
\Biggl(
\gamma'_{(a)}(i)
+ e^{-\lambda} \sum_{j \in [N] \setminus i}
\gamma'_{(a)}(j)
+ e^{-\lambda} \sum_{b \in [K] \setminus a}
\biggl(
\gamma'_{(b)}(i)
+ e^{-\lambda} \sum_{j \in [N] \setminus i}
\gamma'_{(b)}(j)
\biggr)
\Biggr),
\end{align}
which has $O(N^2 K^2 T)$ complexity.
The continuous observation case of $v_{xq}$ is
\begin{align}
\frac{1}{T T'}
\sum_{i=1}^N
\Biggl(
\biggl\langle
\sum_{s=1}^T \gamma_s(i) \phi(x_s)
,
\sum_{t=1}^{T'} \gamma'_t(i) \phi(x'_t)
\biggr\rangle
+
e^{-\lambda}
\sum_{j=1}^N
\biggl\langle
\sum_{s=1}^T \gamma_s(i) \phi(x_s)
,
\sum_{t=1}^{T'} \gamma'_t(j) \phi(x'_t)
\biggr\rangle
\Biggr)
\label{eqn:continuous-vxq-explicit}
\\
= \frac{1}{T T'}
\Biggl(
\biggl(
\sum_{s,t}
k(x_s, x'_t)
\sum_{i=1}^N
\gamma_s(i) \gamma'_t(i)
\biggr)
+
e^{-\lambda}
\biggl(
\sum_{s,t}
k(x_s, x'_t)
\sum_{i,j}
\gamma_s(i) \gamma'_t(j)
\biggr)
\Biggr).
\label{eqn:continuous-vxq-implicit}
\end{align}
Whereas \eqref{eqn:continuous-vxq-implicit} is $O(N^2 T^2)$, this can be reduced by explicitly representing RKHS elements in \eqref{eqn:continuous-vxq-explicit}.
Defining $\xi(i,j) := \frac{1}{T-1} \sum_{t \in T} \xi_t(i,j)$, the kernel on the random variable clique $(q_t, q_{t+1})$ $v_{qq}$ is
\begin{align}
\sum_{ \substack{i \in [N] \\ j \in [N]} }
\xi(i,j)
\Biggl(
\xi'(i,j)
+
e^{-\lambda} \sum_{j' \in [N] \setminus j} \xi'(i,j')
+
e^{-\lambda}
\sum_{i' \in [N] \setminus i} \biggl(
\xi'(i',j)
+
e^{-\lambda} \sum_{j' \in [N] \setminus j} \xi'(i',j')
\biggr)
\Biggr).
\end{align}
Interestingly, as $\lambda$ approaches infinity, the LMMK on HMMs takes a form similar to the plain Fisher kernel on HMMs \cite{jaakkola1999egm} (i.e., when the Fisher information matrix is replaced by the identity matrix).
From our results it will appear that the differences in the computation of two kernels significantly affect their performance.

\section{Related Work} \label{sec:relation}
\subsection{Probability product kernel}
For probability measures $p$ and $p'$ and $x \in \X$, the probability product kernel \cite{jebara2004ppk} is
\begin{align}
\ppk(p, p') = \int_\X p(x)^\rho p'(x)^\rho dx.
\end{align}
The probability product kernel (PPK) is a special case of the GMMK. 

For the case where $\rho = 1$, the following Proposition easily follows.
\begin{proposition}
The probability product kernel with $\rho = 1$ is a special case of the generative mean map kernel with convergence exponential in $\lambda$ as $\lambda \rightarrow \infty$.
\end{proposition}
\begin{proof}
We use the convergence of the scaled\footnote{This scaling is only for the technical reason of the limit in the last line of the proof.} Gaussian kernel $k_G(x,x) = \sqrt{\lambda} \exp(-\lambda \|x - x'\|^2)$ to the identity kernel $k_\delta(x,x') = \delta(x - x')$ as $\lambda \rightarrow \infty$.
The PPK $\displaystyle \lim_{\lambda \rightarrow \infty} \mmk(p, p')$ expands to
\begin{align*}
\lim_{\lambda \rightarrow \infty} \int_\X \int_\Xp p(x) p'(x') k_G(x,x') dx \, dx'
= \lim_{\lambda \rightarrow \infty} \int_\X \int_\Xp p(x) p'(x') \sqrt{\lambda} \exp(-\lambda \|x - x'\|^2) \, dx \, dx'.
\end{align*}
Now, using $(x,x') \mapsto (a,b) := (x, x - x')$ and making the substitution $\sigma := \frac{1}{\sqrt{\lambda}}$ yields
\begin{align*}
\lim_{\sigma \rightarrow 0} \int_{\mathcal{A}} \int_{\mathcal{B}} p(a) p'(a - b) \frac{1}{\sigma} \exp(-\|b\|^2/\sigma^2) \, da \, db = \int_{\mathcal{A}} p(a) p'(a) da = \ppk(p, p'). & \qedhere
\end{align*}

\end{proof}
Further, it is possible to express a GMMK analog to the PPK for the case of $\rho \neq 1$; however, the expectation operator is fundamental to the GMMK's derivation and this operator requires $\rho=1$. Provided that the GMMK can be computed for the distribution of the observed random variables, graphical models for which the PPK is computable are also computable for the GMMK; this can be seen by observing that the Gaussian kernel couples each pair of observed variables $x_i$ and $x'_i$ into a 2-clique. In the clique graph used by the junction tree algorithm, this 2-clique appears wherever $x_i$ appears in the PPK's clique graph.

Another connection between the GMMK $k_m(p,p')$ and the PPK $k_p$ is that $k_m$ is an expectation of $k_p$ between one fixed distribution and an isotropic Gaussian centered at the points drawn from the other distribution. For $\sigma = (2 \lambda)^{-1/2}$, we easily see that
\begin{align}
\E_{x \sim p'}[k_p(p,\mathcal{N}(x, \sigma^2))] = \E_{x \sim p}[k_p(p',\mathcal{N}(x, \sigma^2))].
\end{align}

\subsection{Other related kernels}
As mentioned earlier, the concept of embedding probability distributions into Hilbert spaces via the expectation of a feature map dates back to Guilbart \cite{guilbart1979pse}. The concept was explored further by Suquet \cite{suquet1995dem,suquet2008rkh}\footnote{Guilbart and Suquet actually consider signed measures.}. Hein and Bousquet \cite{hein2005hmp} later generalized this kernel by replacing the scalar product of the measures by a positive definite function between the measures; however, their work did not discuss the learning framework for the case of structured data such as images and time series, and so far the complexity of the associated RKHS has not been quantified in terms amenable for generalization error guarantees for empirical risk minimization. The expectation of a feature map is similar to the marginalized kernel \cite{tsuda2002mkb}, although that work does not discuss RKHSs and  focuses on count kernels rather than higher order (Taylor type) kernels such as the Gaussian RBF kernel. The derivation of the GMMK also is different, coming from the direction of Hilbert space embeddings of distributions to arrive at a kernel with good regularization properties.

The EMMK relies upon observing the labels in order to include them in a mean map of the full joint distribution over observed variables and labels (latent variables). Without learning a generative model with latent state variables, the EMMK is limited to graphical model dependencies between only the observed variables (see the discussion in Section \ref{sec:latent-MMK}).

To our knowledge, the Bhattacharyya affinity is the earliest form of a similarity measure between probability distributions \cite{bhattacharyya1943mdb}, whereas the Fisher kernel \cite{jaakkola1999egm} is the earliest one used as a machine learning method. The kernel is based on the score $\nabla_{\theta} \log p(X \mid \hat{\theta})$ for $\hat{\theta}$ the maximum likelihood estimate of a model. The Fisher kernel is sensitive to the parameterization of the statistical family used and is not easily computable without using a surrogate for the Fisher information matrix. The heat kernel \cite{lafferty2005dks} is not sensitive to the parameterization, but it is rarely computable in closed form; however, it can be approximated under certain conditions to yield a positive definite kernel, e.g. by using leading terms in the parametrix expansion for small time parameter $t$. Although it has not yet been shown, the heat kernel may be (approximately via the parametrix expansion) computable for HMMs with multinomial observation distributions.

Rational kernels \cite{cortes2004rational} are another class of kernels that can handle variable-length sequences. Unlike rational kernels, the GMMK is not restricted to observations constituted by a finite alphabet (as is clear from the GMMK on probability distributions over $\real^n$).

A possible weakness of kernels which do not take make use of the structure of a probability space is that they treat all distributions with disjoint support identically. It is not difficult to show that the Bhattacharyya affinity and general probability product kernels produce a similarity of 0 for distributions with disjoint support. For the case of kernels based off of the Kullback-Leibler-divergence, a similar result trivially holds: 
\begin{align*}
KL(\theta_1 || \theta_2)
= E_{\theta_1} \log \frac{p(x \mid \theta_1)}{p(x \mid \theta_2)}
&= \int_{\X} p_{\theta_1}(x) \log(p_{\theta_1}(x)) dx  - \int_{\X} p_{\theta_1}(x)  \log(p_{\theta_2}(x)) dx \\
&= -H(p_{\theta_1}) + \log 0 + 0 \rightarrow \infty.
\end{align*}
Incorporating smoothness into a similarity measure can alleviate this problem. In fact, we can show that the smoothing property can provide theoretical guarantees with respect to complexity of the RKHS.

\section{Statistical Learning Bounds} \label{sec:theory}

To better understand the GMMK $k_m$ from a statistical learning theory perspective, we first explore the complexity of the hypothesis space induced by this kernel. 
For simplicity of the analysis, we restrict our analysis to kernels on symmetric, univariate location-family probability distributions with a bounded location parameter.
This restriction ensures that the kernel is translation invariant with respect to the location family parameter $\theta \in \Theta \subset \real$, where $\Theta$ is compact. Recall that location families are parameterized by a location parameter $\theta$, so we have that the density function $f_\theta(x) = f_0(x - \theta)$. With this definition, we can view the kernel as a function $k: \Theta \times \Theta \rightarrow \real$, where $\Theta \subset \real$ is the space of the location family parameter.

\subsection{Covering number bounds on the RKHS}
Prior to bounding the complexity of the hypothesis space induced by the GMMK, we need to introduce some notation. Let $\hilb$ be the RKHS corresponding to the generative mean map kernel $k$ restricted to the above location family domain. Denote by $B_\hilb(R)$ the radius $R$ ball of $\hilb$:
\begin{align}
B_\hilb(R) := \{f \in \hilb: \|f\|_\hilb \leq R \},
\end{align}
and let $B_\X := B_\X(1)$ be the unit ball in some metric space $\X$.
We define $I_K$ to be the inclusion
\begin{align}
\id: \hilb \rightarrow C(X).
\end{align}
Recall that the $\epsilon$-covering number $\mathcal{N}(M, \epsilon)$ of a subset $M$ of a metric space $(X,d)$ is defined as the minimal $n \in \mathbb{N}$ such that there exist $n$ balls in $X$ of radius $\epsilon$ that cover $M$; more formally we have
\begin{align}
\mathcal{N}(X, d, \epsilon) := \inf \{n \in \mathbb{N}: \exists \{x_1, \ldots, x_n\} \subset X \text{ satisfying } M \subseteq \bigcup_{i=1}^n B_d(x_i,\epsilon) \},
\end{align}
where (with mild notation abuse) $B_d(x_0, \epsilon) := \{x \in X: d(x, x_0) \leq \epsilon \}$.
To simplify notation, we may express $\epsilon$-covering numbers as $\mathcal{N}(E, \epsilon)$ for a Banach space $E$, where the metric is taken to be corresponding norm for $E$.
Finally, for an operator $T: E \rightarrow F$ between two Banach spaces $E$ and $F$, the covering numbers of the $T$ are defined as
\begin{align}
\mathcal{N}(T, \epsilon) := \mathcal{N}(T B_E, \epsilon).
\end{align}

For convolutional kernels, Zhou \cite{zhou2002cnl} showed that when $F[k]$ (i.e. the Fourier transform of the convolutional form of $k$ shown below) decays exponentially, the covering number of this operator when acting on a radius-$R$ ball in $\hilb$ satisfies
\begin{align}
\label{eqn:covering-number-bound}
\ln \mathcal{N}(\overline{I_K(B_\hilb(R))}, \|\cdot\|_\infty, \eta) \leq c_{k,n} \left( \ln \frac{R}{\eta} \right)^{n+1}
\end{align}
for a constant $c_{k,n}$ depending only on the kernel and the dimension $n$ of the domain. In our case, $n = 1$.

We first show that the Fourier transform indeed indeed does decay exponentially for $k$, and then we briefly summarize a result from Steinwart and Christmann \cite{steinwart2008svm} to arrive at a generalization error bound for SVM learning.

Let $\sigma = (2 \lambda)^{-1/2}$ the variance of the Gaussian RBF function used as the base kernel in $k$. We have the following
\begin{theorem}
The Fourier transform of the the convolutional form of $k$ decays exponentially as
\begin{align}
F[k](\omega) = P(\omega)^2 \exp(-\omega^2 \sigma^2/2). \label{eqn:exp-decay}
\end{align}
\end{theorem}
\begin{proof}
We first show that the kernel is convolutional:
\begin{align*}
k(p_\theta, p_{\theta'})
&= \int \int p_\theta(x) p_{\theta'}(x') \frac{1}{\sigma} \exp(-\frac{1}{2}(x-x')^2/\sigma^2) \, dx \, dx' \\
&= \int \int p_\theta(x) p_\theta(x' - \Delta\theta) \frac{1}{\sigma} \exp(-\frac{1}{2}(x-x')^2/\sigma^2) \, dx \, dx' = \tilde{k}(\Delta\theta),
\end{align*}
where $\Delta\theta = \theta' - \theta$. Prior to taking the Fourier transform of $\tilde{k}$, we make a few simplifications:
\begin{align*}
\tilde{k}(\Delta\theta)
&= \int \int p_\theta(x) \frac{1}{\sigma} \exp(-\frac{1}{2}(x'-x)^2/\sigma^2) \, dx \, p_\theta(x' - \Delta\theta) \, dx' \\
&= \int f(x') g(\Delta\theta - x') \, dx',
\end{align*}
where $f(x') := \int p_\theta(x) \frac{1}{\sigma} \exp(-\frac{1}{2}(x'-x)^2/\sigma^2) \, dx$ and $g(\Delta\theta - x') := p_\theta(\Delta\theta - x')$.

In taking the Fourier transform $F[\tilde{k}](\omega) := \mathrm{FT}[\tilde{k}(\theta' - \theta')]$, we apply the convolution theorem to yield $F[k](\omega) = F(x') G(x')$. Now, observe that
\begin{align*}
F(x') :&= \mathrm{FT} \left[ \int p_\theta(x) \frac{1}{\sigma} \exp(-\frac{1}{2}(x'-x)^2/\sigma^2) dx \right] && \\
&= \mathrm{FT} \left[ p_0(x) \right] \mathrm{FT} \left[ \frac{1}{\sigma} \exp(-\frac{1}{2}x^2/\sigma^2) \right] && \textit{(Convolution Theorem)} \\
&= \mathrm{FT}[p_0(x)] \exp(-\omega^2 \sigma^2/2). && \textit{(FT of Gaussian)}
\end{align*}

Define $G(x') := \mathrm{FT}[g(\Delta\theta - x')] = \mathrm{FT}[g(x')]$.
Then \eqref{eqn:exp-decay} follows by defining $P(\omega) := FT[p_0(x)]$.
Hence, the Fourier transform of the kernel operator decays exponentially, independent of the decay of  $P(\omega)$.
\qedhere
\end{proof}

In many cases, we cannot assume that $P(\omega)$ decays rapidly; for the uniform distribution the decay horrifically is not a decay at all (the Fourier transform is periodic). Even for the Laplace distribution, the decay is only quadratic in $\omega$; specifically, the decay is $O(2 \sigma / (a^2 + \omega^2))$. Hence, the Gaussian convolution is critical to the exponential decay; in contrast, the Fourier spectrum of the PPK depends on the Fourier transforms of the distributions \cite{jebara2003bel}. Although yet to be shown, we conjecture that the smoothing via the Gaussian convolution in $k_m$ affords strong regularization properties for more general classes of distributions.

\subsection{SVM Learning bound}

We now show how the above result can be plugged into an oracle inequality for SVM learning with our kernel $k$, by making use of the covering number bound in \eqref{eqn:covering-number-bound}. Define $Y := \{-1,1\}$. Let $\R_{L,Q}$ be the risk defined as $\E_{(x,y) \sim Q}[L(x,y,f(x)]$, for a probability measure $Q$ on $\Theta \times Y$ and a loss function $L: \Theta \times Y \times \real$. Recall that the SVM problem is to minimize the regularized risk functional
\begin{align}
\inf_{\lambda,f \in \hilb} \lambda \|f\|_\hilb^2 + \R_{L,D}(f),
\end{align}
and define
\begin{align}
\R^*_{L,P,\hilb} := \inf_{f \in \hilb} \R_{L,P}(f).
\end{align}

We now can use an oracle inequality such as Theorem 6.25 of Steinwart and Christmann \cite{steinwart2008svm}, restated in slightly different and simplified terms here for convenience:
\begin{theorem}
Denote by $f_{D,\lambda}$ unique minimizer of the SVM problem with regularization term $\lambda$ over empirical distribution $D$.
Let $\Theta \subset \real$ be a compact metric space and $L: \Theta \times Y \times \real \rightarrow [0, \infty]$ be the hinge loss\footnote{Note that the hinge loss is Lipschitz continuous with Lipschitz constant 1 and satisfies $L(\theta, y, 0) \leq 1$ for all $(\theta, y) \in \Theta \times Y$.} $L(x,y,f(x)) = \max\{ 0, 1 - y f(x) \}$. Further, let $\hilb$ denote the RKHS of $k$ and $P$ be a probability measure on $\Theta \times Y$, with $D \sim P^n$ (iid). Then, for fixed $\lambda > 0$, $n \geq 1$, $\epsilon > 0$, and $\tau > 0$, we have with probability greater than $1 - e^{-\tau}$ that
\begin{align*}
\lambda \|f_{D,\lambda}\|_\hilb^2 + \R_{L,P}(f_{D,\lambda}) - \R^*_{L,P,\hilb}
\mbox{ \large $<$ } \,  A_2(\lambda) + 4 \epsilon + (\lambda^{-1/2} + 1)
\sqrt{\frac{2 \tau + 2 \ln (2 \mathcal{N}(B_\hilb, \|\cdot\|_\infty, \lambda^{1/2} \epsilon))}{n}},
\end{align*}
where the approximation error functional $A_2(\lambda) := \inf_{f \in \hilb} \lambda \|f\|_\hilb^2 + \R_{L,P}(f) - \R^*_{L,P,\hilb}$ approaches zero as $\lambda \rightarrow 0$.
\end{theorem}

Finally, we note that for the kernel of interest, under the restriction specified at the beginning of this section, we have the covering number bound
\begin{align}
\ln \mathcal{N}(B_\hilb, \|\cdot\|_\infty, \lambda^{1/2} \epsilon) = \ln \mathcal{N}(I_K, \lambda^{1/2} \epsilon) \leq c_k \left( \ln \frac{1}{\lambda^{1/2} \epsilon} \right)^2,
\end{align}
for a constant $c_k$ depending only on the kernel.

\section{Experimental Results} \label{sec:experiments}
\subsection{Sequence data} We evaluated the GMMK and LMMK through classification of two discrete observation sequence datasets (synthetic and human DNA sequence data) and five continuous observation sequence datasets from the UCR time series dataset repository\footnote{We selected all but one of the 2-class problems from the UCR set. The excluded dataset, Lightning-2, would have required significant preprocessing. No preprocessing aside from scaling was performed.}. We explored the performance of SVMs using the GMMK and the PPK
($\mmk$ for $\lambda \rightarrow \infty$)
on discrete observation HMMs,
as well as the latent MMK using a model learned from one class.
All HMMs were learned via uniform distribution initialization for initial state and state transition probabilities, random initialization of emission probabilities for discrete observations and constrained $k$-means clustering initialization for continuous observations, a segmental clustering update via the Viterbi algorithm \cite{rabiner1989thm}, and execution of the Baum-Welch algorithm until the first of log-likelihood convergence within $10^{-6}$ or 1000 iterations.

For the GMMK we used the expectation of the Gaussian RBF kernel, where in the discrete case each symbol is mapped to a vector using the 1-in-$k$ encoding, and we explored logarithmically spaced settings of the parameter $\lambda$.
For the GMMK we experimented with linearly spaced settings of the witness length $T$. For all kernels, the regularization parameter $C$ was tested at logarithmically spaced values. We compare these results with a maximum-likelihood Bayes classifier which maintains an HMM model for each class, the Gaussian kernel using the metric induced by the Fisher kernel with a model for one class, and the EMMK using Markov models of orders\footnote{An order-$k$ Markov model induces maximal cliques of size $k+1$, with the kernel described in \cite{zhang2009kmi}. An EMMK with this model effectively is a string kernel whose feature space representation consists of counts of each string of length $k+1$.} 1, 2, 3, and 4. We report all results for stratified 10-fold cross-validation.

The synthetic dataset consists of 2 classes. For each class, we manually constructed a 3-state 2-symbol HMM to serve as a sequence generator. 500 binary sequences of length 100 were generated for each class. All HMMs learned had 3 states.
Table \ref{tab:methods-loss} shows the Bayes classifier and the GMMK perform almost equally, whereas the Fisher kernel and empirical/latent MMKs
perform slightly worse. The PPK exhibits much higher loss than the other methods, possibly owing to its higher sensitivity to errors in model estimation.

We also ran two-class sequence classification experiments on a random subset of 500 exons and 500 introns from the HS$^3$D dataset \cite{pollastro2002hhs}.
The sequence lengths vary from order of $10^2$ to $10^4$ symbols in length, making for a challenging classification problem.
For the GMMK and PPK, we adopted a heuristic formula \cite{jebara2004ppk} to choose the number of states $n$ in the HMM learned for a particular sequence: $n = \lfloor\frac{1}{2}\sqrt{k^2 + 4 (T \gamma + k + 1)} - \frac{1}{2} k\rfloor + 1$, for $k$ symbols and a constant $\gamma$ set to 0.1. For the Bayes classifier, we used a 4-state model for the exons and an 8-state model for the introns because this configuration produced the lowest loss.
For
the latent MMK and
the Fisher kernel, we used a 4-state model trained on the exons. For the EMMK, we report results for a Markov dependency model of order $3$, which produced the best results among orders $\{1, 2, 3, 4\}$.

We restrict our results for the GMMK to the setting of $T = 30$ which had the lowest mean loss over all values of $\lambda$. Our results in Table \ref{tab:methods-loss} show that the GMMK performs slightly better than the PPK, both of which outperform the Fisher kernel. The GMMK significantly outperforms the Bayes classifier and the EMMK, and the LMMK performs relatively poorly on this problem.
\begin{table*}[t]
\caption{\label{tab:methods-loss} \small Error for methods on sequence data. $^*$ indicates that a class-balanced random subsample was used. ``X'' indicates that experiments were not performed due to computational intractability.}
\small
\begin{center}
\begin{tabular}{|l l|l|l|l|l|l|l|l|l|}
\hline
\# observations &           & GMMK              & PPK ($\rho = 1$)   & PPK ($\rho = 0.5$) & EMMK                         & LMMK       & Fisher            & Bayes \\
\hline
\tiny(1000) & Synthetic   & 0.063               & 0.091                   & 0.247                     & 0.067\footnotemark  & 0.068        & 0.066             & \textbf{0.062} \\
\hline
\tiny(1000) & HS$^3$D   & \textbf{0.144}  & 0.149                   & 0.165                     & 0.215                         & 0.279        & 0.170            & 0.192 \\
\hline
 \tiny(56)    & Coffee       & 0.232                & \textbf{0.214}     &  0.268                    & 0.286                         & X               & 0.268            & 0.339 \\
\hline
\tiny(200)   & ECG           & 0.250                & 0.250                   & \textbf{0.240}       & 0.320                         & X               & 0.335            & 0.335 \\
\hline
\tiny(200)   & Gun-Point & \textbf{0.165}  & 0.230                  &  0.265                     & 0.185                         & X               & 0.385            & 0.235 \\
\hline
\tiny(1524$^*$) & Wafer & \textbf{0.077}  & 0.231                  &  0.144                    & 0.173                          & X               & 0.104             & 0.245 \\
\hline
\tiny(1000$^*$) & Yoga  & \textbf{0.304}  & \textbf{0.304}    &  0.341                    & 0.340                          & X               & 0.305             & 0.422 \\
\hline
\end{tabular}
\end{center}
\end{table*}

\footnotetext{EMMK results on synthetic data were unstable with respect to $\lambda$.}
\begin{figure}[t]
 \begin{center}
     \label{fig:species-manifold}
      \includegraphics[width=75mm]{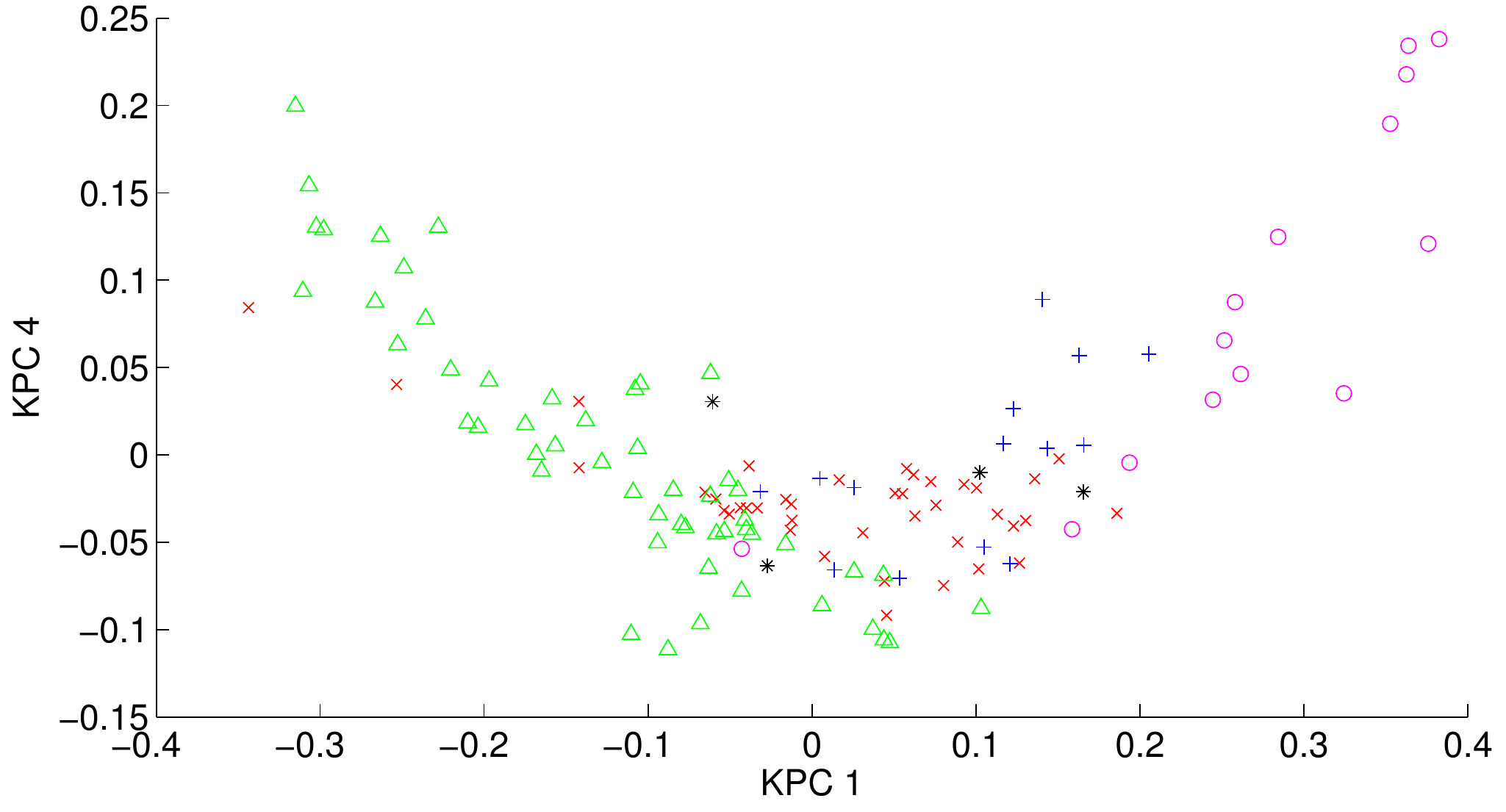}
 \end{center}
  \caption{\label{fig:hs3d-loss} \small
1$\st$ vs 4$\nth$ kernel principal components (kPCs) of the GMMK matrix on INBio data for $\lambda=1$. Plants, molluscs, arachnids, insects, and fungi are triangles, circles, asterisks, x's, and crosses respectively. The plot is similar for the 1$\st$ vs 2$\nd$ kPCs (not shown).}
\end{figure}

Across all of the time series datasets, the GMMK is competitive with or outperforms the other methods.
On Gun-Point and Wafer, the GMMK well outperforms the rest of the methods. On the other three time series datasets it is at least competitive.
Interestingly, the EMMK often performs well. In the future, we plan to explore the performance of its generalization, the latent mean map kernel.

\subsection{Biodiversity data}
We also applied the GMMK to visualize ecological relationships between species sampled in Costa Rica by INBio\footnote{The National Biodiversity Institute of Costa Rica, \protect\url{http://www.inbio.ac.cr/en}.}. For each species, observations consist of counts on a grid, where each grid location has the
accompanying abiotic features altitude, median annual temperature, annual precipitation, isothermality, and temperature seasonality.
Because estimating spatial density is very suspectible to the fact that this is presence data, we instead estimate density in the abiotic feature space.
Each species has a true density in this 5-dimensional space, which we estimate using Gaussian RBF KDEs.
For each species sampling, we optimized the bandwidth of a KDE by minimizing the integrated square error \cite{wasserman2004as}.
After obtaining the optimal bandwidth for each species sampling, we formed the kernel matrix between all pairs of the species KDEs using the GMMK with $\lambda = 1$. We then applied kPCA to identify whether particular components well-separate certain groups of species. The visualization of the 1$\st$ and 4$\nth$ kernel principal components in Figure \ref{fig:species-manifold} shows some separation between the different groups: the insects and arachnids are clustered near the center, with plants toward the left and molluscs toward the right.

\section{Conclusion} \label{sec:conclusion}
We have discussed results on two generative kerenls, each embodying a different family of smoothed similarity measures between probability distributions. We have shown the first learning theoretic guarantees for the generative mean map kernel under certain conditions, and we have introduced the latent mean map kernel, a natural extension to Hilbert space embeddings of empirical distributions. Unlike other generative kernels, the GMMK both can operate on latent variable models and enjoy certain generalization error bounds independent of the Fourier spectra of the distributions it is comparing. We hope to generalize this result to a larger class of distributions in the future. Experimental results on discrete and continuous data suggest that the GMMK can outperform and is at least competitive with the PPK, the EMMK, the Fisher kernel, and the Bayes classifier.
Using sampling techniques \cite{jebara2004ppk}, it is straightforward to extend the GMMK to more expressive models such as factorial HMMs and Markov random fields.
For models with continuous random variables, such as HMMs with mixture of Gaussians observation distributions, evaluating the LMMK is highly expensive computationally; further testing in the continuous domain demands that we first overcome these computational challenges by using approximations of explicit representations of RKHS elements.

{
  \bibliographystyle{plain}
  \bibliography{mmk_v1}

\begin{thebibliography}{10}

\bibitem{altun2004efc}
Y.~Altun, A.J. Smola, and T.~Hofmann.
\newblock {Exponential families for conditional random fields}.
\newblock In {\em Proceedings of the 20th conference on Uncertainty in
  Artificial Intelligence}. AUAI Press Arlington, Virginia, United States,
  2004.

\bibitem{bhattacharyya1943mdb}
A.~Bhattacharyya.
\newblock {On a measure of divergence between two statistical populations
  defined by their probability distributions}.
\newblock {\em Bull. of the Calcutta Math. Soc.}, 35, 1943.

\bibitem{cortes2004rational}
C.~Cortes, P.~Haffner, and M.~Mohri.
\newblock {Rational kernels: Theory and algorithms}.
\newblock {\em Journal of Machine Learning Research}, 5, 2004.

\bibitem{gretton2008kmt}
A.~Gretton, K.~M. Borgwardt, M.~Rasch, B.~Sch\"olkopf, and A.~J. Smola.
\newblock {A kernel method for the two-sample-problem}.
\newblock {\em Journal of Machine Learning Research}, 1:1--10, 2008.

\bibitem{gretton2005msd}
A.~Gretton, O.~Bousquet, A.~Smola, and B.~Sch\"olkopf.
\newblock {Measuring Statistical Dependence with Hilbert-Schmidt Norms}.
\newblock {\em Lecture Notes in CS}, 3734, 2005.

\bibitem{guilbart1978eps}
C.~Guilbart.
\newblock {\em {\'E}tude des produits scalaires sur l'espace des mesures.
  {E}stimation par projection. {T}ests \`a noyaux}.
\newblock PhD thesis, Lille 1, France, 1978.

\bibitem{guilbart1979pse}
C.~Guilbart.
\newblock Produits scalaires sur l'espace des mesures.
\newblock {\em Annales de l'Institut Henri Poincar{\'e}}, 15:333--354, 1979.

\bibitem{hein2005hmp}
Matthias Hein and Olivier Bousquet.
\newblock Hilbertian metrics and positive definite kernels on probability
  measures.
\newblock In Robert~G. Cowell and Zoubin Ghahramani, editors, {\em Artificial
  Intelligence and Statistics 2005}, pages 136--143. Society for Artificial
  Intelligence and Statistics, 2005.

\bibitem{jaakkola1999egm}
T.~S. Jaakkola and D.~Haussler.
\newblock {Exploiting generative models in discriminative classifiers}.
\newblock {\em Advances in Neural Information Processing Systems}, 1999.

\bibitem{jebara2003bel}
T.~Jebara and R.~Kondor.
\newblock {Bhattacharyya and Expected Likelihood Kernels}.
\newblock {\em Learning Theory and Kernel Machines}, pages 57--71, 2003.

\bibitem{jebara2004ppk}
T.~Jebara, R.~Kondor, and A.~Howard.
\newblock {Probability product kernels}.
\newblock {\em Journal of Machine Learning Research}, 5, 2004.

\bibitem{kondor2003kbs}
R.~Kondor and T.~Jebara.
\newblock {A kernel between sets of vectors}.
\newblock In {\em International Conference on Machine Learning}, 2003.

\bibitem{lafferty2005dks}
J.~Lafferty and G.~Lebanon.
\newblock {Diffusion Kernels on Statistical Manifolds}.
\newblock {\em Journal of Machine Learning Research}, 6, 2005.

\bibitem{pollastro2002hhs}
P.~Pollastro and S.~Rampone.
\newblock {HS3D: Homo Sapiens Splice Site Dataset}.
\newblock {\em Nucleic Acids Research, Annual Database Issue}, 2002.

\bibitem{rabiner1989thm}
L.~R. Rabiner.
\newblock A tutorial on hidden {M}arkov models and selected applications in
  speech recognition.
\newblock {\em Proceedings of the IEEE}, 77(2), 1989.

\bibitem{rahimi2007rfl}
A.~Rahimi and B.~Recht.
\newblock {Random features for large-scale kernel machines}.
\newblock {\em Advances in Neural Information Processing Systems}, 2007.

\bibitem{roweis1999url}
S.~Roweis and Z.~Ghahramani.
\newblock {A unifying review of linear Gaussian models}.
\newblock {\em Neural computation}, 11(2):305--345, 1999.

\bibitem{scholkopf2002lk}
B.~Sch\"olkopf and A.~J. Smola.
\newblock {\em Learning with kernels}.
\newblock MIT press Cambridge, Mass, 2002.

\bibitem{scholkopf1997kpc}
B.~Sch\"olkopf, A.J. Smola, and K.R. Muller.
\newblock {Kernel principal component analysis}.
\newblock {\em Lecture Notes in CS}, 1327, 1997.

\bibitem{smola2007hse}
A.~Smola, A.~Gretton, L.~Song, and B.~Sch\"olkopf.
\newblock A {H}ilbert space embedding for distributions.
\newblock {\em Lecture Notes in CS}, 4754, 2007.

\bibitem{steinwart2002ikc}
I.~Steinwart.
\newblock {On the influence of the kernel on the consistency of support vector
  machines}.
\newblock {\em Journal of Machine Learning Research}, 2, 2002.

\bibitem{steinwart2008svm}
I.~Steinwart and A.~Christmann.
\newblock {\em {Support vector machines}}.
\newblock Springer, 2008.

\bibitem{steinwart2006edr}
I.~Steinwart, D.~Hush, and C.~Scovel.
\newblock {An explicit description of the reproducing kernel Hilbert spaces of
  Gaussian RBF kernels}.
\newblock {\em IEEE Transactions on Information Theory}, 52(10):4635, 2006.

\bibitem{suquet1995dem}
Charles Suquet.
\newblock {Distances euclidiennes sur les mesures sign{\'e}es et applicationa
  des th{\'e}oremes de Berry-Ess{\'e}en}.
\newblock {\em Bulletin of the Belgian Mathematical Society}, 2:161--181, 1995.

\bibitem{suquet2008rkh}
Charles Suquet.
\newblock Reproducing kernel {H}ilbert spaces and random measures.
\newblock In {\em More Progresses in Analysis: Proceedings of the 5th
  International Isaac Congress}, page 143. World Scientific Pub Co Inc, 2008.

\bibitem{tsuda2002mkb}
K.~Tsuda, T.~Kin, and K.~Asai.
\newblock {Marginalized kernels for biological sequences}.
\newblock {\em Bioinformatics}, 18(Suppl 1), 2002.

\bibitem{tsybakov2008ine}
A.B. Tsybakov.
\newblock {\em {Introduction to nonparametric estimation}}.
\newblock Springer, 2008.

\bibitem{wasserman2004as}
L.~Wasserman.
\newblock {All of statistics}.
\newblock {\em Statistics}, 2004.

\bibitem{xu2006ecr}
J.~W. Xu, P.~P. Pokharel, K.~H. Jeong, and J.~C. Principe.
\newblock {An explicit construction of a reproducing Gaussian kernel Hilbert
  space}.
\newblock In {\em ICASSP 2006 Proceedings}, volume~5, 2006.

\bibitem{zhang2009kmi}
X.~Zhang, L.~Song, A.~Gretton, and A.~Smola.
\newblock Kernel measures of independence for non-iid data.
\newblock {\em Advances in Neural Information Processing Systems}, 2009.

\bibitem{zhou2002cnl}
D.X. Zhou.
\newblock {The covering number in learning theory}.
\newblock {\em Journal of Complexity}, 18(3):739--767, 2002.

\end{thebibliography}
}

\end{document}